\documentclass{article}
\usepackage{amsmath,amsfonts,amssymb}
\usepackage[latin1]{inputenc}
\usepackage{colt09e,times}

\newlength{\minipagewidth}
\renewcommand{\phi}{\varphi}
\renewcommand{\epsilon}{\varepsilon}
\newcommand{\cA}{\mathcal{A}}
\newcommand{\cB}{\mathcal{B}}
\newcommand{\cS}{\mathcal{S}}
\newcommand{\cX}{\mathcal{X}}
\newcommand{\cV}{\mathcal{V}}
\newcommand{\cY}{\mathcal{Y}}
\newcommand{\E}{\mathbb{E}}

\newcommand{\ind}{\mathbb{I}}

\newcommand{\bp}{{\bf p}}
\newcommand{\nchoosek}[2]{\ensuremath{\left( \begin{array}{c} #1 \\ #2 \end{array} \right)}}
\newcommand{\lj}{\ell^{(j)}}

\newcommand{\bX}{\mathbf{X}}
\newcommand{\by}{\mathbf{y}}
\newcommand{\bx}{\mathbf{x}}
\newcommand{\ba}{\mathbf{a}}
\newcommand{\ok}{\mbox{\small \textsc{ok}}}
\renewcommand{\d}{\mbox{d}}
\renewcommand{\leq}{\leqslant}
\renewcommand{\geq}{\geqslant}

\newtheorem{exam}{Example}

\newenvironment{expar}[1]{\begin{exam}[#1] \rm}{\end{exam}}
\newtheorem{remark}{Remark}

\newtheorem{ass}{Assumption}

\newtheorem{proposition}[theorem]{Proposition}

\title{Online Multi-task Learning with Hard Constraints}
\author{G{\'a}bor Lugosi\thanks{Supported by the Spanish Ministry of Science and Technology
grant MTM2006-05650 and by the PASCAL Network of Excellence under EC grant {no.} 506778} \\
ICREA and Universitat Pompeu Fabra \\
Barcelona, Spain \\
{lugosi@upf.es}
\And
Omiros Papaspiliopoulos\thanks{Supported by the Spanish Ministry of Science and Technology
under a ``Ramon y Cajal'' scholarship} \\
Universitat Pompeu Fabra \\
Barcelona, Spain \\
{omiros.papaspiliopoulos@upf.edu}
\And
Gilles Stoltz\thanks{Supported by the French National Research Agency (ANR)
under grants JCJC06-137444 ``From applications to theory in learning and adaptive statistics''
and 08-COSI-004 ``Exploration--exploitation for efficient resource allocation'',
and by the PASCAL Network of Excellence under EC grant {no.} 506778
} \\
Ecole Normale Sup{\'e}rieure, CNRS \\
Paris, France \\
HEC Paris, CNRS, \\
Jouy-en-Josas, France \\
{gilles.stoltz@ens.fr}
}

\begin{document}

\maketitle

\begin{abstract}
We discuss multi-task online learning when a decision maker has to
deal simultaneously with $M$ tasks.  The tasks are related, which is
modeled by imposing that the $M$--tuple of actions taken by the
decision maker needs to satisfy certain constraints.  We give natural
examples of such restrictions and then discuss a general class of
tractable constraints, for which we introduce computationally
efficient ways of selecting actions, essentially by reducing to an
on-line shortest path problem.  We briefly discuss ``tracking'' and
``bandit'' versions of the problem and extend the model in various
ways, including non-additive global losses and uncountably infinite
sets of tasks.
\end{abstract}

\section{Introduction}

Multi-task learning has recently received considerable attention,
see \cite{DeLoSi07,AbBaRa07,Men08,CaCeGe08}.
In multi-task learning problems, one simultaneously learns several
tasks that are related in some sense.
The relationship of the tasks has been modeled in different ways in the literature.
In our setting, a decision maker chooses an action simultaneously for each of $M$ given tasks,
in a repeated manner. (To each of these tasks corresponds a game, and we will use
interchangeably the concepts of game and task.)
The relatedness is accounted for by putting some hard constraints on these simultaneous actions.

As a motivating example, consider a distance-selling company
that designs several commercial offers for its numerous
customers, and the customers are ordered (say) by age. The company has to
choose whom to send which offer.  A loss of earnings is suffered whenever
a customer does not receive the commercial offer that would have
been best for him. Basic marketing considerations suggest that
offers given to customers with similar age should not be very
different, so the company selects a batch of offers that satisfy such
a constraint. Additional budget constraint may limit further the set of
batches from which the company may select.
After the offers are sent out, the customers' responses are observed
(at least partially) and
new offers are selected and sent. We model such situations by playing
many repeated games simultaneously with the restriction that
the vector of actions that can be selected at
a time needs to belong to a previously given set.
This set in determined beforehand by the budget and marketing constraints
discussed above.
The goal of the decision maker is to minimize the total
accumulated regret (across the many games and through time), that is,
perform, on the long run, almost as well as the best constant vector
of actions satisfying the constraint.
\medskip

The problem of playing repeatedly several games simultaneously has
been considered by \cite{Men08} who studies convergence to Nash
equilibria but does not address the issue of computational feasibility
when a large number of games is played.  On-line multi-task learning
problems were also studied by
\cite{AbBaRa07} and \cite{DeLoSi07}. As the latter reference,
we consider minimizing regret simultaneously in parallel,
by enforcing however some hard constraints.
As \cite{AbBaRa07}, we measure the total loss as the sum of the
losses suffered in each game but assume that all tasks have to be performed
at each round.  (This assumption is, however relaxed in Section \ref{sec:other},
where we consider global losses more general than the sums of losses.)
The main additional difficulty
we face is the requirement that the decision maker chooses from a
restricted subset of vectors of actions. In previous models
restrictions were only considered on the comparison class, but not on
the way the decision maker plays.
\medskip

We formulate the problem in the framework of on-line regret
minimization, see \cite{CeLu06} for a survey. The main challenge is
to construct a strategy for playing the many games simultaneously with
small regret such that the strategy has a manageable computational
complexity. We show that in various natural examples the computational
problem may be reduced to an online shortest path problem in an
associated graph for which well-known efficient algorithms exist.
(We however propose a specific scheme for implementation that is slightly
more effective.)

The results can be extended easily to the ``tracking'' case in which
the goal of the decision maker is to perform as well as the best
strategy that can change the vector of actions (taken from the restricted set)
at a limited number of times.
We also consider the ``bandit'' version of the problem when the decision
maker, instead of observing the losses of all actions in all games,
only learns the sum of the losses of the chosen actions.

Finally, we also consider cases when there are infinitely
many tasks, indexed by real numbers. In such cases the decision maker
chooses a function from a certain restricted class of functions. We
show examples that are natural extensions of the cases we consider
for finitely many tasks and discuss the computational issues that are
closely related to the theory of exact simulation of continuous-time
Markov chains.

We concentrate on exponentially weighted average
forecasters because, when compared to its most likely competitors,
that is, follow-the-leader-type algorithms, they have better
performance guarantees, especially in the case of bandit feedback.
Besides, the two families of forecasters, as pointed out by
\cite{AbBaRa07}, usually have implementation complexities of the same order.

\section{Setup and notation}

In the simplest model studied in this paper,
a decision maker deals simultaneously with $M$ tasks, indexed
by $j = 1,\ldots,M$. For simplicity,
we assume that all games share the same finite action space
$\cX = \{ x_1,\ldots,x_N \} \subset \mathbb{R}$.
(Here, we do not identify actions with integers but with real numbers, for
reasons that will be clear in Section~\ref{sec:motiv}.)

To each tasks $j=1,\ldots,M$ there is an associated outcome space
$\cY_j$ and a loss function $\lj: \cX \times \cY_j \to [0,1]$.
We denote by $\bx = \bigl( x_{k_1},\ldots,x_{k_M} \bigr)$
the elements of $\cX^M$ and call them vectors of simultaneous actions.
The tasks are played repeatedly and at each round $t =1,2,\ldots$, the
decision maker chooses a vector $\bX_t =(X_{1,t}\ldots,X_{M,t}) \in
\cX^M$ of simultaneous actions.
(That is, he chooses indexes $K_{1,t},\ldots,K_{M,t} \in \{ 1,\ldots,N\}$
and $X_{j,t} = x_{K_{j,t}}$ for all $j = 1,\ldots,M$.)
We assume that the choice of $\bX_t$ can be made at random,
according to a probability distribution over $\cX^N$ which will
usually be denoted by $\bp_t$.
The behavior of the opponent player
among all tasks is described by the vector of outcomes $\by_t =
(y_{1,t},\ldots,y_{M,t})$.

We are interested in the loss suffered by the decision maker and
we do not assume any specific probabilistic or strategic behavior of
the environment. In fact, the outcome vectors $\by_t$, for $t=1,2,\ldots$,
can be completely arbitrary and we measure the performance of the
decision maker by comparing it to the best of a class of reference
strategies. The total loss suffered by the decision maker at time $t$
is just the sum of the losses over tasks:
\[
\ell(\bX_t,\by_t) = \sum_{j=1}^M \lj(X_{j,t},y_{j,t})~.
\]
The important point is that the decision maker has some restrictions to
be obeyed in each round, which we also call hard constraints.  They
are modeled by a subset $\cA$ of the set of possible simultaneous
actions $\cX^M$; the forecaster is only allowed to play vectors
$\bX_t$ in $\cA$. This subset $\cA$ captures the relatedness among the
tasks.

The decision maker aims at minimizing his regret, defined by the
difference of his cumulative loss with respect to the cumulative loss
of the best constant vector of actions, determined in hindsight,
among the set of allowed vectors $\cA$.
Formally, the regret is defined by
\[
R_n
= \sum_{t=1}^n \ell(\bX_t,\by_t) - \min_{\bx \in \cA} \sum_{t=1}^n \ell(\bx,\by_t)~.
\]
In the basic, {\sl full information}, version of the problem the
decision maker, after choosing $\bX_t$,
observes the vector of outcomes $\by_t$. In the {\sl bandit} setting,
only the total loss $\ell(\bX_t,\by_t)$ becomes available to the
decision maker.

Observe that in the case of $M = 1$ task, the problem reduces to the
well-studied problem of ``on-line prediction with expert advice''
or ``sequential regret minimization,'' see \cite{CeLu06} for the history
and basic results. This is also the case when $M \geq 2$ but
$\cA = \cX^M$, since the decision maker could then treat each task independently
from others and maintain $M$ parallel forecasting schemes, at least
in the full-information setting. Under the bandit assumption the
problem becomes the ``multi-task bandit problem'' discussed in
\cite{CeLu09}, which is also easy to solve by available techniques.
However, when $\cA$ is a proper subset of $\cX^M$, interesting
computational problems arise.
The efficient implementation we propose requires a condition the set
$\cA$ of restrictions needs to satisfy. This structural condition,
satisfied in several natural examples discussed below, permits
us to reduce the problem to the well-studied problem of predicting
as well as the best path between two fixed vertices of a graph.

In order to make the model meaningful, just like in the most basic versions
of the problem, we allow the decision maker to randomize its decision in each
period.
More formally, at each round of the repeated game, the decision maker
determines a distribution on $\cX^M$ (restricted to
the set $\cA$) and draws the action vector $\bX_t$ according
to this distribution.
Before determining the outcomes, the opponent may have access
to the probability distribution the decision maker uses but not to the
realizations of the random variables.

\subsection*{Structure of the paper}
We start by stating some natural examples on which the proposed techniques
will be illustrated.
We then study the full-information version of
the problem (when the decision maker observes all past outcomes
before determining his probability distribution)
by proposing first a hypothetical scheme with good performance
and then stating an efficient implementation of it.

We also consider various extensions. One of them is the bandit setting, when
only the sum of losses of the chosen simultaneous actions are observed.
Another extension is the ``tracking'' problem when, instead of competing
with the best constant vector of actions, the decision maker intends
to perform as well as the best strategy that is allowed to switch a certain
limited number of times (but always satisfying the restrictions).
We also consider alternative global loss functions that do not necessarily
sum the losses over the tasks. Finally, we describe
a setting in which there are infinitely many tasks
indexed by an interval. This is a natural extension of the main examples
we work with and the algorithmic problem has some interesting connections
with exact simulation of continuous-time discrete Markov chains.

\section{Motivating examples}
\label{sec:motiv}

We start by describing four examples that we will be able to handle
with the proposed machinery. The examples are defined by their
corresponding sets $\cA \subset \cX^M$ of permitted simultaneous
actions.

\begin{expar}{Internal coherence}
\label{ex:Intcoh}
Assume that tasks are linearly ordered and any
two consecutive tasks, though different, share some
similarity. Therefore, it is a natural requirement that the actions taken
in two consecutive games be not too far away from each other.
One may also interpret this as
a matter of internal coherence of the decision maker.
To model this, we assume that the actions are ranked in the action set $\cX$
according to some logic and impose some maximal dissimilarity $\gamma >0$
between the actions of two consecutive tasks, that is,
\[
\cA = \Bigl\{ \bigl( x_{k_1},\ldots,x_{k_M} \bigr): \ \ \forall \, j \leq M-1, \ \, \bigl| x_{k_j} - x_{k_{j+1}} \bigr|
\leq \gamma \Bigr\}~.
\]
\end{expar}

\begin{expar}{Escalation constraint}
\label{ex:Esc}
Once again we assume that the tasks are linearly ordered and the
actions are ranked. Imagine that tasks correspond to consumers and
that the higher the index of the task, the more favorable the
conditions for the consumer (and the higher the loss of earnings of
the seller, who is the decision maker). The constraint
decision maker has to satisfy is that higher-ranked costumers
need to receive better conditions, at least
within the same round of play.
That is, the simultaneous actions must form a non-decreasing
sequence in the following sense,
\[
\cA = \Bigl\{ \bigl( x_{k_1},\ldots,x_{k_M} \bigr): \ \ \forall \, j \leq M-1, \ \, k_j \leq k_{j+1} \Bigr\}~.
\]
\end{expar}

\begin{expar}{Constancy constraint}
\label{ex:Const}
Assume that tasks are ordered and that the decision maker
should not vary its action too often.
This is measured by the fact that the decision maker
must stick to an action for
several consecutive tasks and that he can shift to a new action only
at a limited number $m$ of tasks, which we model by
\[
\cA = \left\{ \bigl( x_{k_1},\ldots,x_{k_M} \bigr): \ \
\sum_{j=1}^{M-1} \mathbb{I}_{ \{ k_j \ne k_{j+1} \} } \leq m \right\}~.
\]
\end{expar}

\begin{expar}{Budget constraint}
\label{ex:Budget}
Here we assume that the number $x_{k_j}$ associated to action $k$ in
task $j$ represents the cost of choosing this action.
The freedom of the decision maker is limited
by a budget constraint.
For example, one may face a situation when
the decision maker has a constant budget $B$ to be used at each round, that is,
\[
\cA = \left\{ \bigl( x_{k_1},\ldots,x_{k_M} \bigr): \ \ \sum_{j=1}^M x_{k_j} \leq B \right\}~.
\]
To make things more concrete, we assume, in this example only,
that $x_k = k$. One should then take for $B$ as an integer between $M$ and $NM$.
For smaller values $\cA$ becomes empty and for larger values $\cA = \cX^N$.
\end{expar}

\section{Exponentially weighted averages}
\label{sec:perf}

By considering each element of $\cA$ as a (meta-)expert, we can reduce
the problem to the usual single-task setting and exhibit a forecaster
with a good performance bound that, in its straightforward
implementation, has a computational cost proportional to the
cardinality of $\cA$.

More precisely, for each round $n \geq 1$, we denote by
\[
L_n(\bx) = \sum_{t=1}^n \ell(\bx,\by_t)
\]
the cumulative loss of the simultaneous actions $\bx \in \cX$, and
define an instance of the exponentially weighted average forecaster on
these cumulative losses. That is, at round $t=1$, the decision maker
draws an element $\bX_1$ uniformly at random in $\cA$ and for each
round $t \geq 2$, draws $\bX_t$ at random according to the
distribution $\bp_t$ on $\cA$ which puts the following mass on each
$\bx \in \cA$,
\begin{equation}
\label{eq:pt}
\bp_t(\bx) = \frac{\exp \bigl( - \eta L_{t-1}(\bx) \bigr)}{\sum_{\ba \in \cA}
\exp \bigl( - \eta L_{t-1}(\ba) \bigr)}~,
\end{equation}
where $\eta>0$ is a parameter to be tuned. The bound follows from a
direct application of well-known results, see, for instance,
\cite[Corollary~4.2]{CeLu06}.

\begin{proposition}
\label{th:full}
For all $n \geq 1$,
the above instance of the exponentially weighted average forecaster, when run with $\eta = (1/M) \, \sqrt{8 (\ln N) / n}$,
ensures that for all $\delta > 0$, its regret is bounded, with probability at most $1-\delta$, as
\[
R_n \leq M \left( \sqrt{\frac{n \ln |\cA|}{2}} + \sqrt{\frac{n}{2} \ln \frac{1}{\delta}} \right)
\]
where $|\cA|$ denotes the cardinality of $\cA$.
\end{proposition}

The computational complexity of this forecaster, in its naive
implementation, is proportional to $|\cA|$, which is prohibitive in
all examples of Section~\ref{sec:motiv}
since the cardinality of $\cA$ is exponentially large. For example,
in Example~\ref{ex:Intcoh}, if we denote by
\[
\rho = \min \biggl\{ \Bigl| \bigl\{ x' \in \cX : \,\, |x - x'| \leq \gamma \bigr\} \Bigl| : \ x \in \cX \biggr\}
\]
a common lower bound on the number of $\gamma$--close actions to any action in $\cX$, then
\[
|\cA| \geq N \rho^{M-1}~.
\]
In Example~\ref{ex:Esc},
by first choosing the $m$ actions to be used (in increasing order) and
the $m-1$ corresponding shift points, one gets
\begin{eqnarray*}
|\cA| & = & \sum_{m=1}^N \nchoosek{N}{m} \nchoosek{M+m-1}{m-1} \\
& \geq & \sum_{m=1}^N \nchoosek{N}{m} \frac{M^{m-1}}{(m-1)!} \geq \frac{(M+1)^N}{(N-1)!}~.
\end{eqnarray*}
In the case of at most $m$ shifts
in the simultaneous actions, discussed in Example~\ref{ex:Const},
we have
\[
|\cA| \geq \nchoosek{M+m}{m} N(N-1)^m
\]
(where the lower bound is obtained by considering only the simultaneous actions
with exactly $m$ shifts). That is, $|\cA|$ is of the order of $(MN)^m/m!$.
Finally, with the budget constraint of Example~\ref{ex:Budget},
the typical size of $\cA$ is exponential in $M$, as
\[
|\cA| \geq \rho^M
\]
where $\rho = \lfloor B/M \rfloor$ is the lower integer part of $B/M$.

\section{Efficient implementation with online shortest path}
\label{sec:effHMM}

In this section we show how the computational problem of
drawing a random vector of actions $\bX_t \in \cA$ according to the
exponentially weighted average distribution can be reduced to
the well-studied online shortest path problem.
Recall that in the online
shortest path problem (see, e.g., \cite{TaWa04,GyLiLu04,GyLiLu05})
the decision maker selects, at each round of the game,
a path between two given vertices (the {\sl source} and the {\sl sink})
in a given graph.
A loss is assigned to each edge of the graph in every round of
the game and the loss of
a path is the sum of the losses of the edges.
A path can be selected  according to the
exponentially weighted average distribution in a computationally
efficient way by a dynamic programming-type algorithm, see \cite{TaWa04}
or \cite[Section~5.4]{CeLu06}.
The algorithm has complexity $O(|{\cal E}|)$ where $\cal E$ is the
set of edges of the graph.

We first explain how the problem of drawing a joint action in the
multi-task problem can be reduced to an online shortest path problem in
all the examples presented above and then indicate how to efficiently
sample from the distribution $\bp_t$ defined in (\ref{eq:pt}).

\subsection{A Markovian description of the constraints}

In order to define the corresponding graph in which the online
shortest path problem is equivalent with our hard-con-strained
multi-task problem, we introduce a set $\cS$ of hidden {\sl states}.
The value of the hidden state controls that the hard
constraints are satisfied along the sequence of simultaneous
actions. To this end, denote by $S$ the state function, which, given a
vector of actions (of length $\leq M$), outputs the corresponding state
in $\cS$.

We also consider an additional state $\star$ meaning that
the hard constraint is not satisfied.  We denote $\cS^\star = \cS \cup
\{ \star \}$.  By definition,
\[
\cA = \bigl\{ \bx \in \cX^M : \ S(\bx) \ne \star \bigr\}~.
\]
To make things more concrete we now describe $\cS$ and $S$
on all four examples introduced in Section~\ref{sec:motiv}.
\medskip

The first two examples are the simplest as all the information is
contained in the current action; their hidden state space $\cS$ is
reduced to a single state $\ok$.  For Example~\ref{ex:Intcoh}, for all
sequences $\bigl( x_{k_1},\ldots,x_{k_j} \bigr)$ of length $1 \leq j
\leq M$, one defines
\begin{multline}
\nonumber
S \Bigl( \bigl(x_{k_1},\ldots,x_{k_j} \bigr) \Bigr) \\
= \left\{
\begin{array}{ll}
\ok & \mbox{if for all} \ i \leq j-1, \ \ \bigl| x_{k_i} - x_{k_{i+1}} \bigr| \leq \gamma, \\
\star & \mbox{otherwise},
\end{array}
\right.
\end{multline}
whereas for Example~\ref{ex:Esc},
\begin{multline}
\nonumber
S \Bigl( \bigl(x_{k_1},\ldots,x_{k_j} \bigr) \Bigr) \\
= \left\{
\begin{array}{ll}
\ok & \mbox{if for all} \ i \leq j-1, \ \ x_{k_i} \leq x_{k_{i+1}}, \\
\star & \mbox{otherwise}.
\end{array}
\right.
\end{multline}

In Example~\ref{ex:Const} the underlying hidden state counts
the number of shifts seen so far in the sequence of actions, so
$\cS = \{ 0,\ldots,m \}$ and for all sequences
$\bigl(x_{k_1},\ldots,x_{k_j} \bigr)$ of length less or equal to $M$,
we first define
\[
S' \Bigl( \bigl(x_{k_1},\ldots,x_{k_j} \bigr) \Bigr) = \sum_{j=1}^{M-1} \mathbb{I}_{ \{ {k_j} \ne {k_{j+1}} \} }
\]
and then
\begin{multline}
\nonumber
S \Bigl( \bigl(x_{k_1},\ldots,x_{k_j} \bigr) \Bigr) \\
= \left\{
\begin{array}{ll}
S' \Bigl( \bigl(x_{k_1},\ldots,x_{k_j} \bigr) \Bigr) & \mbox{if} \ S' \Bigl( \bigl(x_{k_1},\ldots,x_{k_j} \bigr) \Bigr) \leq m, \\
\star & \mbox{otherwise}.
\end{array}
\right.
\end{multline}
Finally, in Example~\ref{ex:Budget}, the hidden state monitors the
budget spent so far, that is, $\cS = \{ 0,\ldots,B \}$,
\[
S' \Bigl( \bigl(x_{k_1},\ldots,x_{k_j} \bigr) \Bigr) = \sum_{i=1}^j x_{k_i}~,
\]
and
\begin{multline}
\nonumber
S \Bigl( \bigl(x_{k_1},\ldots,x_{k_j} \bigr) \Bigr) \\
= \left\{
\begin{array}{ll}
S' \Bigl( \bigl(x_{k_1},\ldots,x_{k_j} \bigr) \Bigr) & \mbox{if} \ S' \Bigl( \bigl(x_{k_1},\ldots,x_{k_j} \bigr) \Bigr) \leq B, \\
\star & \mbox{otherwise}.
\end{array}
\right.
\end{multline}
\medskip

In view of these examples, the following assumption on $S$ is natural.
\begin{ass}
\label{ass:1}
The state function is Markovian in the following sense.
For all $j \geq 2$ and all vectors $\bigl(x_{k_1},\ldots,x_{k_j} \bigr)$, the state
$S \Bigl( \bigl(x_{k_1},\ldots,x_{k_j} \bigr) \Bigr)$ only depends
on the value of $x_{k_j}$ and on the state $S \Bigl( \bigl(x_{k_1},\ldots,x_{k_{j-1}} \bigr) \Bigr)$.
\end{ass}

We further assume that there exists a transition function $T$
that, to each pair $(x,s)$ (corresponding to some task $j$) formed by
an action $x \in \cX$ and a hidden state $s \in \cS^\star$, associates
pairs $(x',s') \in \cX \times \cS$ (to be used in task $j+1$). Put
differently, $T \bigl( (x,s) \bigr)$ is a subset of $\cX \times
\cS^\star$ that indicates all legal transitions.  We impose that when
the prefix of a sequence is already in the dead end state $s = \star$,
the whole sequence stays in $\star$, that is, for all $x \in\cX$,
\[
T \bigl( (x,\star) \bigr) = \cX \times \{ \star \}~.
\]
Once again, to make things more concrete, we describe $T$ for the
four examples introduced in Section~\ref{sec:motiv}.

Example~\ref{ex:Intcoh} relies on $\cS = \{ \ok \}$ and the transitions
\[
T \bigl( (x,\ok) \bigr) = \bigl( \cX \cap [x-\gamma,\,x+\gamma] \bigr) \times \{ \ok \}
\]
for all $x \in \cX$.
Example~\ref{ex:Esc} can be modeled with $\cS = \{ \ok \}$ and the transitions
\[
T \bigl( (x,\ok) \bigr) = [x,x_N] \times \{ \ok \}~.
\]
for all $x \in \cX$.

For Example~\ref{ex:Const}, the transition function is given by
\[
T \bigl( (x,s) \bigr) = \{ (x,s) \} \cup \Bigl( \bigl( \cX \setminus \{ x \} \bigr) \times \{ s+1 \} \Bigr)
\]
for all $s = 0,\ldots,m-1$ and
\[
T \bigl( (x,m) \bigr) = \{ (x,m) \} \cup \Bigl( \bigl( \cX \setminus \{ x \} \bigr) \times \{ \star \} \Bigr)
\]
for $s = m$.

Finally, the one of Example~\ref{ex:Budget} is given by
\[
T \bigl( (x,s) \bigr) = \left\{ \begin{array}{ll}
    \cX \times \{ s+x \} & \mbox{if $s+x \leq B$,} \\
\cX \times \{ \star \} & \mbox{if $s+x > B$.}
                            \end{array}  \right.
\]

\subsection{Reduction to an online shortest path problem}

We are now ready to describe the graph by which a
constrained multi-task problem can be reduced to
an online shortest path problem.
Assume that $\cA$ is such that there is a corresponding state space
$\cS$, a state function $S$ satisfying Assumption \ref{ass:1},
and a transition function $T$.
We define the cumulative losses $L^{(j)}_{n}$ suffered in
each task $j = 1,\ldots,M$ between rounds $t=1$ and $n$ as follows.
For all $x \in \cX$,
\[
L^{(j)}_n(x) = \sum_{t=1}^n \ell^{(j)}(x,y_{j,t})~.
\]
Of course, with the notation above, for all $n \geq 1$
and all $\bx = \bigl( x_{k_1},\ldots,x_{k_j} \bigr)$,
\[
L_n(\bx) = \sum_{j=1}^M L^{(j)}_n \bigl( x_{k_j} \bigr)~.
\]
In the sequel, we extend the notation by convention to $n = 0$, by
$L_0 \equiv 0$ and $L^{(j)}_0 \equiv 0$ for all $j$.

Then, for each round $t = 1,\ldots,n$,
we define a directed acyclic graph with at most $MN|\cS|$ vertices.
Each vertex corresponds to task-action-state triple $(j,x_k,s)$,
where $j=1,\ldots,M$, $k=1,\ldots,N$, and $s\in \cS$.
Two vertices $v=(j,x_k,s)$ and $v'=(j',x_{k'},s')$
are connected with a directed edge if and only if
$j'=j+1$, and $(x_{k'},s') \in T(x_k,s)$, that is,
$(x_k,s) \to (x_{k'},s')$ is a legal transition between tasks $j$ and
$j+1$. The loss associated to such an edge equals
$L^{(j')}_{t-1}(x_{k'})$, the cumulative loss
of action $x_{k'}$ in task $j'$ in the previous time rounds.
We also add two vertices, the ``source'' node
$u_0$ and the ``sink'' $u_1$ as follows. There is a directed edge between
$u_0$ and every vertex of the form $(1,x_k,s)$ with $k=1,\ldots,N$
and $s\neq \star$. Its associated losses equal
$L^{(1)}_{t-1}(x_{k})$.
Finally, every vertex of the form
$(M,x_k,s)$ with $k=1,\ldots,N$
and $s \neq \star$ is connected to the sink $u_1$ with edge loss $0$.

In the graph defined above, choosing a path between the source and the sink
is equivalent to choosing a legal $M$--tuple of actions in the multi-task
problem.
(Note that there is no path between $u_0$ and $u_1$
containing a vertex with $s=\star$.)
The sum of the losses over the edges of a path is just
the cumulative loss of the corresponding $M$--tuple of actions.
Generating a legal random $M$--tuple
according to the exponentially weighted average
distribution is thus equivalent to generating a random path in this
graph according to the exponentially weighted average distribution.
This can be done with a computational complexity of the order of the number of edges
defined above, see, e.g., \cite[Section~5.4]{CeLu06}.
In our case, since edges only connect two consecutive tasks, the
number of edges is at most $1 + M N^2 |\cS|^2$.  In
Section~\ref{sec:cplx} we discuss the number of edges and the related
complexity on the examples of Section~\ref{sec:motiv}.

Since edges only exist between consecutive tasks,
the above implementation by reduction to an online shortest path problem
takes a simple form, which we detail below for concreteness.
It will be useful to have it for Section~\ref{sec:GlobalLoss}.

\subsection{Brief recall of the way the efficient implementation goes}
\label{sec:effsampl}

In order to generate a random $M$--tuple
of actions according to the distribution $\bp_t$,
we first rewrite the probability distribution $\bp_t$
in terms of the state function $S$
and the cumulative losses $L^{(j)}_{t-1}$ suffered in
each task $j$.
To do so, we denote by $\delta_\bx$ the Dirac mass on $\bx =
\bigl( x_{k_1},\ldots,x_{k_M} \bigr)$, that is,
the probability distribution
over $\cX$ that puts all probability mass on $\bx$.
The definition (\ref{eq:pt}) then rewrites as
\begin{multline}
\label{eq:pt2}
\bp_t = \\
\sum_{\bx \in \cX^N}
\frac{\mathbb{I}_{ \{ S(\bx) \ne \star \} } \exp \Bigl( - \eta \sum_{j=1}^M L^{(j)}_{t-1}\bigl(x_{k_j} \bigr) \Bigr)}{
\sum_{\ba \in \cX^N} \mathbb{I}_{ \{ S(\ba) \ne \star \} } \exp \Bigl( - \eta \sum_{j=1}^M L^{(j)}_{t-1}\bigl(a_{k_j} \bigr) \Bigr)}
\,\, \delta_{\bx}~.
\end{multline}
\medskip

Before proceeding with the random generation of vectors $\bX_t$
according to $\bp_t$, we introduce an auxiliary sequence of weights
and explain how to maintain it.  For all rounds $t \geq 0$, tasks $j
\in \{ 1,\ldots,M \}$, actions $x \in \cX$, and states $s \in \cS$, we
define
\begin{eqnarray*}
\lefteqn{ w_{t,j,x,s} = } \\
& \displaystyle{\sum_{x_{k_1},\ldots,x_{k_{j-1}} \in \cX}} & \exp \left( - \eta \left( L^{(j)}_t (x) + \sum_{i = 1}^{j-1} L^{(i)}_t
\bigl( x_{k_{i}} \bigr) \right) \right) \\
& & \ \ \times \ind_{ \bigl\{ S ( i_{k_1},\ldots,i_{k_{j-1}},x ) = s \bigr\} }~.
\end{eqnarray*}
Note that we do not consider the state $\star$ here.

Now, for all rounds $t \geq 0$, actions $x \in \cX$, and states
$s \in \cS$, one simply has
\[
w_{t,1,x,s} = \exp \left( - \eta L^{(1)}_t (x) \right) \, \ind_{ \{S(x) = s \} }~.
\]
Then, an induction (on $j$) using Assumption~\ref{ass:1} shows that
for all $1 \leq j \leq M-1$, actions $x' \in \cX$, and states
$s' \in \cS$,
\begin{multline}
\label{eq:updtw}
w_{t,j+1,x',s'} = \\
\displaystyle{\sum_{x \in \cX, \, s \in \cS}}
w_{t,j,x,s} \, \ind_{ \{ (x',s') \in T ( (x,s) ) \} } \,
\exp \left( - \eta L^{(j+1)}_t (x') \right) ~.
\end{multline}
\medskip

We now show how to use these weights to sample from the desired
distribution $\bp_t$, for $t \geq 1$.  We proceed in a backwards
manner, drawing first $X_{M,t}$, then, conditionally to the value of
$X_{M,t}$, generating $X_{M-1,t}$, and so on, till $X_{1,t}$.

To draw $X_{M,t}$, we note that
equation (\ref{eq:pt2}) shows that the $M$--th marginal induced by $\bp_t$ is
the distribution over $\cX$ that puts a probability mass proportional to
\[
\sum_{s \in \cS} \,\, w_{t-1,M,k,s}
\]
on each action $x \in \cX$. It is therefore easy to generate a random element $X_{M,t}$
with the appropriate distribution. We actually need to draw a pair $(X_{M,t}, S_{M,t}) \in \cX \times \cS$
distributed according to the distribution on $\cX \times \cS$ proportional to the $w_{t-1,M,k,s}$.

We then aim at drawing the actions (and hidden states) corresponding
to the previous tasks according to the (conditional) distribution
$\bp_t \bigl( \,\cdot\, | \, X_{M,t}, S_{M,t} \bigr)$ on $(\cX \times
\cS)^{M-1}$.  Again by using the Markovian assumption on $S$, it turns
out that the $(M-1)$--th marginal of this distribution on $\cX \times
\cS$ is proportional, for all pairs $(x,s) \in \cX \times \cS$, to
\[
w_{t,M-1,x,s} \,\,  \ind_{ \{ (X_{M,t},S_{M,t}) \in T ( (x,s) ) \} }~.
\]
This procedure, based on conditioning by the future, can be repeated
to draw conditionally all the actions $X_{1,t},X_{2,t},\ldots,$
$X_{M,t}$ and hidden state spaces $S_{1,t},S_{2,t},\ldots,S_{M,t}$.
In particular, we use, to draw $X_{j,t}$ and $S_{j,t}$, the
distribution on $\cX \times \cS$ proportional to
\begin{equation}
\label{eq:gener}
w_{t,j,x,s} \,\,  \ind_{ \{ (X_{j+1,t},S_{j+1,t}) \in T ( (x,s) ) \} }~.
\end{equation}
The realization $\bX_t = (X_{1,t},X_{2,t},\ldots,X_{M,t})$ obtained
this way is indeed according to the distribution $\bp_t$.

\subsubsection{Complexity of this procedure for the considered examples}
\label{sec:cplx}

The space complexity is of the order of at most $O \bigl( M N |\cS| \bigr)$, since
weights have to be stored for all ask-action-state triples.
The computational complexity, at a given task, for performing
the updates (\ref{eq:updtw}) for all $x'$ and $s'$ is bounded
by the number of pairs $(x',s')$ times the maximal number
of pairs $(x,s)$ that lead to $(x',s')$. We denote by $T_{\max}$
this maximal number of transitions. Then, the complexity
of performing (\ref{eq:updtw}) for all tasks is bounded
by $O \bigl( M N |\cS| T_{\max} \bigr)$.
The complexity of the random generations (\ref{eq:gener})
is negligible in comparison, since it is of the order of $O \bigl( M N |\cS| \bigr)$.

We now compute $T_{\max}$
for the four examples
described in Section \ref{sec:motiv} and
summarize the complexity results (both for the efficient and the naive implementations)
in the table below.
In Example~\ref{ex:Intcoh}, in addition to the parameter $\rho$ introduced in Section~\ref{sec:perf},
we consider a common upper bound on the number of $\gamma$--close actions to any action in $\cX$,
\[
\vartheta = \max \biggl\{ \Bigl| \bigl\{ x' \in \cX : \,\, |x - x'| \leq \gamma \bigr\} \Bigl| : \ x \in \cX \biggr\}~.
\]
Then, $T_{\max} = \vartheta$.
In Example~\ref{ex:Esc}, the value $T_{\max} = N$ is satisfactory.
In Example~\ref{ex:Const}, only $T_{\max} = N$ pairs $(x,s)$, of the form
$x=x'$ and $s=s'$ or $x \ne x'$ and $s' = s+1$, can lead to $(x',s')$.
A similar argument shows that in the case
of Example~\ref{ex:Budget}, only $T_{\max} = N$ such transitions are possible also.
\begin{center}
\begin{tabular}{lll}
\hline
{Ex.} & Efficient & Naive \\
\hline
\ref{ex:Intcoh}. & $M N \vartheta$ & $\geq N \rho^{M-1}$ \\
\ref{ex:Esc}. & $M N^2$ & $\geq (M+1)^N/(N-1)!$ \\
\ref{ex:Const}. & $M N^2 m$ & $\geq (MN)^m/m!$ \\
\ref{ex:Budget}. & $M N^2 B$ & $\geq (B/M)^M$ \\
\hline
\end{tabular}
\end{center}

\section{Tracking}

In the problem of {\sl tracking the best expert}
of \cite{HW98,Vov99}, the goal of the forecaster is,
instead of competing with the best fixed action, to compete with
the best sequence of actions that can switch actions a limited
number of times. We may formulate the tracking problem in
the framework of multi-task learning with hard constraints.
In this case, just like before, at each time $t$,
the decision maker chooses an $M$--tuple of actions from the set
$\cA$ of legal vectors. However, now regret is measured by
comparing the cumulative loss of the forecaster $\sum_{t=1}^n \ell(\bX_t.\by_t)$
with
\[
    \min_{(\bx_1,\ldots,\bx_n) \in \Sigma_K(\cA)} \sum_{t=1}^n \ell(\bx_t,\by_t)
\]
where $\Sigma_K(\cA)$ is the set of all sequences of vectors of $\cA$
that may switch values at most $K$ times (i.e., the time interval
$1\ldots,n$ can be divided into at most $K+1$ intervals such that
over each interval the same $M$--tuple of actions). In this case
it is well known that exponentially weighted average over the
class $\Sigma_K(\cA)$ of meta-experts (see \cite[Sections 5.5 and 5.6]{CeLu06}
for a statement of the results and precise bibliographic references)
yields a regret
\begin{multline}
\nonumber
\sum_{t=1}^n \ell(\bX_t,\by_t)-
\min_{(\bx_1,\ldots,\bx_n) \in \Sigma_K(\cA)} \sum_{t=1}^n \ell(\bx_t,\by_t) \\
= O\left( M\sqrt{n\left(K\ln |\cA| + K \ln \frac{n}{K}\right)}
+ M\sqrt{n \ln \frac{1}{\delta}} \right)
\end{multline}
which holds with probability $1-\delta$.
Moreover, the complexity
of the generation of the $M$--tuples of actions achieving the regret bound above
is bounded, at round $t$, by $O \bigl( t^2 + M N^2 |\cS|^2 \, K t \bigr)$.

\section{Multi-task learning in bandit problems}
\label{sec:bandits}

In this section we briefly discuss a more difficult version of the problem
when the decision maker only observes the total loss $\ell(\bX_t,\by_t)$
suffered though the $M$ games but the sequence $\by_t$
of outcomes remains hidden.
This may be considered as a ``bandit'' variant of the basic problem.

Then our problem becomes an instance of an {\sl online linear optimization}
problem studied by
\cite{AwKl04,McBl04,GyLiLuOt07,DaHaKa08,AbHaRa08,BaDaHaKaRaTe08,CeLu09}.
For example, since the dimension of the underlying
space is given by the number of edges, in number always less than $1 + M N^2 |\cS|^2$,
the results of \cite{DaHaKa08} imply that a variant of
the exponentially weighted average predictor achieves an expected
regret of the order
\begin{multline}
\nonumber
\E \left[ \sum_{t=1}^n \ell(\bX_t,\by_t) \right]
- \min_{\bx \in \cX} \E \left[ \sum_{t=1}^n \ell(\bx,\by_t) \right] \\
= O\left( M \left( M^{3/2} N^3 |\cS|^3 + N |\cS| \sqrt{M} \ln |\cA| \right) \sqrt{n}\right)~.
\end{multline}
\cite{BaDaHaKaRaTe08} proved that an appropriate modification
of the forecaster satisfies this regret bound with high probability.
As the predictor of \cite{DaHaKa08} requires exponentially weighted averages
based on appropriate estimates of the losses, it can be implemented
efficiently with the methods described in Section~\ref{sec:effHMM}.
More precisely, it first computes, at each round $t$, estimates of all
losses $\lj(x,y_{j,t})$, when $x \in \cS$ and $j = 1,\ldots,M$ and then
can use the methods described in Section~\ref{sec:effHMM}.
The computationally most complex point is to compute these estimates,
which essentially relies on computing and inverting an incidence
matrix of size bounded by the number of edges. This can be done
in time $O \bigl( M^2 N^4 |\cS|^4 \bigr)$. Details are omitted.

\section{Other measures of loss}
\label{sec:other}

In this section we study two variations of the multi-task problem
in which the loss of the decision maker in a round is computed
in a way different from summing the losses over the tasks.
consisting in computing in a different manner the total loss incurred
within a round on the $M$ tasks. \cite{DeLoSi07} measure losses
by different norms of the loss vector across tasks but they do not
consider the hard constraints introduced here.

\subsection{Choosing a subset of the tasks}

In our first example, at every round of the game,
the forecaster chooses $m$ out of the $M$ tasks and
only the losses over the chosen tasks count in the total
loss. For simplicity we only consider the full-information
case here when the decision maker has access to all losses
(not only those that correspond to the chosen tasks).

Formally, we add an extra action $-$ which means
that the decision maker does not play in this task. Of course,
$\ell^{(j)}(-,y) = 0$ for all $j$ and $y \in \cY_j$. We model this by
\begin{multline}
\nonumber
\cA = \\
\left\{ \bigl( x_{k_1},\ldots,x_{k_M} \bigr) \in \bigl( \cX \cup \{ - \} \bigr)^M : \ \
\sum_{j=1}^{M} \mathbb{I}_{ \{ x_{k_j} \ne - \} } = m \right\}~.
\end{multline}

Since an element of $\cA$ is characterized by the $m$ tasks (out of
$M$) in which it takes one among the $N$ actions of $\cX$, we have
\[
|\cA| = \nchoosek{M}{m} N^m~.
\]
Here again, the bound of Proposition~\ref{th:full} applies and
an efficient implementation is possible as in Section~\ref{sec:effHMM},
at a cost of $O \bigl( M N^2 m^2 \bigr)$.

Of course, additional hard constraints could be added in this example.

\subsection{Choosing a different global loss}
\label{sec:GlobalLoss}

This paragraph is inspired by~\cite{DeLoSi07} where a notion of a
``global loss function'' is introduced.  The loss measured
$\ell(\bX_t,\by_t)$ in a round is now a given function $\psi$ of the
losses $\ell^{(j)}(X_{j,t},y_{j,t})$ incurred in each task $j$, which
may be different from their sum,
\[
\ell(\bX_t,\by_t) = \psi \Bigl( \ell^{(1)}(X_{1,t},y_{1,t}), \ldots, \ell^{(M)}(X_{M,t},y_{M,t}) \Bigr)~.
\]
Examples include for instance the max-loss or the min loss,
\begin{multline}
\nonumber
\psi(u_1,\ldots,u_M) = \max \{ u_1,\ldots,u_M\}
\\ \mbox{or} \quad
\psi(u_1,\ldots,u_M) = \min \{ u_1,\ldots,u_M\}~,
\end{multline}
whenever one thinks in terms of the best or worst performance.

We make a Markovian assumption on the losses. More precisely, we
assume that they can be computed recursively as follows.  There exists
a function $\varphi$ on $\mathbb{R}^2$ such that, defining the
sequence $(v_2,\ldots,v_M)$ as
\[
v_2 = \varphi(u_1,u_2) \quad \mbox{and} \quad v_t = \varphi(v_{t-1},u_t) \ \ \mbox{for} \ t \geq 3~,
\]
one has
\[
v_M = \psi(u_1,\ldots,u_M)~.
\]
This means that if the values $v_t$ are added as a hidden state space
$\cV$, and if the latter is not too big, computation of the
distributions $\bp_t$ defined, for all rounds $t \geq 0$ and all
simultaneous actions $\bx \in \cA$, by
\[
\bp_t(\bx) = \frac{\exp \left( - \eta \sum_{s=1}^{t-1}
\ell(\bx,\by_t) \right)}{\sum_{\ba \in \cA}
\exp \left( - \eta \sum_{s=1}^{t-1}
\ell(\ba,\by_t) \right)}~,
\]
can be done efficiently (a statement which we will be made more
precise below).  In addition, it is immediate, by reduction to the
single-task setting, that a regret bound as in
Proposition~\ref{th:full} holds, where one simply has to replace $M$
with the supremum norm of $\psi$ over the losses.

We only need to explain how and when the results of
Section~\ref{sec:effsampl} extend to the case considered above. The
state $\cV$ of possible values for the possible sequences of $v_t$
should not bee too large and the update (\ref{eq:updtw}) has to be
modified, in the sense that it is unnecessary to multiply by the
exponential of the losses; the global loss will be taken care of at
the last step only, its value being tracked by the additional hidden space. The
complexity is of the order of at most $O \bigl( M N^2 |\cS|^2 |\cV|^2 \bigr)$.
Examples of small $|\cV|$ include the case when the global loss is a
max-loss or a min-loss and the case when all outcome spaces $\cY_j$ and loss
functions $\ell^{(j)}$ are identical. In this case, $|\cV| = N$.

Note that here, in addition to this change of the measure of the total
incurred in a round, additional hard constraints can still be
considered, since the base state space $\cS$ is designed to take care of them.

\section{Multi-task learning with a continuum of tasks and hard constraints}

In this section we extend our model by considering infinitely
many tasks. We focus on the case when tasks are indexed by the $[0,1]$
interval.
We start by describing the setup, then propose an ideal forecaster
whose exact efficient implementation remains a challenge.
We propose discretization instead,
which will take us back to the previously discussed case of a finite number of
tasks.

\subsection{Continuum of tasks with  a constrained number of shifts}
\label{sec:descrcont}

Assume that tasks are indexed by $g \in [0,1]$. The decision maker
has access to a finite set $\cX = \{ x_1, \ldots, x_N\}$ of actions.
Taking simultaneous actions in all games at a given round $t$ is
now modeled by choosing a measurable function
\[
I_t : g \in [0,1] \mapsto I_t(g) \in \cX~.
\]
The opponent chooses a bounded measurable loss function
 $\psi_t : [0,1] \times \cX \to [0,1]$.
The loss incurred by the decision maker is then given by
\[
\ell_t(I_t) = \int_{[0,1]} \psi_t \bigl( g, I_t(g) \bigr) \,\d g \\
 = \sum_{x \in \cX} \int_{ \{ I_t = x \} } \psi_t(g,x) \,\d g~.
\]
As before, we require that the action of the decision maker satisfies
a hard constraint. One case that is easy to formulate is, that $I_t$
must be right-continuous and the family of actions taken
simultaneously,
\[
\bigl( I_t(g) \bigr)_{g \in [0,1]}
\]
must contain at most a given number $m$ of shifts, where by
definition, there is a shift at $g$ if for all $\varepsilon > 0$, the
set $I_t \bigl( [g-\varepsilon, g] \bigr)$ contains more than two
actions.  We denote by $\cA$ the set of such
simultaneous actions.  Actually, any element of $\cA$ can be described
by its shifts (in number at most $m$), denoted by $g_1,\ldots,g_{m'}$,
with $m' \leq m$, and the actions taken in the intervals
$[g_j,g_{j+1}[$ for all $j = 0,\ldots,m'-1$ where $g_0 = 0$, and on
    $[g_{m'},1]$.

The aim of the decision maker is to minimize the cumulative regret
\[
R_n = \sum_{t=1}^n \ell_t(I_t) - \inf_{I \in \cA} \sum_{t=1}^n \ell_t(I)~,
\]
where the $I_t$ are picked from $\cA$.

\subsection{An ideal forecaster}

We denote by $\mu$ the distribution on $\cA$ induced
by the uniform distribution on $\cX^{m+1} \times [0,1]^m$ via the mesurable application
\begin{multline}
\bigl( x_{k_1},\ldots,x_{k_{m+1}},g_1,\ldots,g_m \bigr) \\ \mapsto
\mathbb{I}_{[0,g_{(1)}[} x_{k_1} + \left( \sum_{j=2}^{m} \mathbb{I}_{[g_{(j-1)},g_{(j)}[} x_{k_{j}} \right)
+ \mathbb{I}_{[g_{(m+1)},1]} x_{k_{m+1}}~,
\end{multline}
where we denoted by $\bigl( g_{(1)},\ldots,g_{(m)} \bigr)$ the order statistics of
the $g_1,\ldots,g_m$. (It is useful to observe for later purposes that
if $G_1,\ldots,G_m$ are {i.i.d.} uniform, then the vector
\begin{multline}
\label{eq:vect}
V(G_1,\ldots,G_m) \\
= \bigl( G_{(1)}, G_{(2)} - G_{(1)}, \ldots, G_{(m)} - G_{(m-1)}, 1 - G_{(m)} \bigr)
\end{multline}
is uniformly distributed over the simplex of probability distributions
with $m+1$ elements.)

For all $t \geq 1$, the ideal forecaster uses probability
distributions $\bp_t$ over $\cA$, defined below, and draws the
application $I_t$ giving the simultaneous actions to be taken at round
$t$ according to $\bp_t$. For $t=1$, we take $\bp_1 = \mu$.  For $t
\geq 2$, we take $\bp_t$ as the probability distribution absolutely
continuous with respect to $\mu$ and with density
\begin{equation}
\label{eq:pt-cont}
\d\bp_t(I) = \frac{\exp \left( - \eta \sum_{s=1}^{t-1} \ell_s(I) \right)}{
\int_\cA \exp \left( - \eta \sum_{s=1}^{t-1} \ell_s(J) \right) \,\d\mu(J)} \, \d\mu(I)~.
\end{equation}
The performance of this forecaster may be bounded as follows.
Note that no assumption of continuity or convexity is needed here.

\begin{theorem}
\label{th:cont}
For all $n \geq 1$,
the above instance of the exponentially weighted average forecaster, when run with
\[
\eta = \sqrt{\frac{8 (m+1) \ln (N \sqrt{n} )}{n}}~,
\]
ensures that for all $\delta > 0$, its regret is bounded, with probability at most $1-\delta$, as
\[
R_n \leq \sqrt{n} \left( 1 + \sqrt{\frac{(m+1) \ln (N \sqrt{n})}{2}} \right) +
\sqrt{\frac{n}{2} \ln \frac{1}{\delta}}~.
\]
\end{theorem}

\begin{proof}
By the Hoeffding-Azuma inequality, since the $\psi_t$
take bounded values in $[0,1]$,
we have that with probability at least $1-\delta$,
\begin{equation}
\label{eq:rncont}
R_n \leq
\sum_{t=1}^n \int_{\cA} \ell_t(I) \,\d\bp_t(I) - \inf_{I \in \cA} \sum_{t=1}^n \ell_t(I)
+ \sqrt{\frac{n}{2} \ln \frac{1}{\delta}}~.
\end{equation}
We denote, for all $t \geq 1$,
\[
W_t = \int_\cA \exp \left( - \eta \sum_{s=1}^{t} \ell_s(I) \right) \,\d\mu(I)
\]
(with the convention $W_0 = 1$).
The bound on the difference in the right-hand side of~(\ref{eq:rncont}) can be obtained
by upper bounding and lower bounding
\[
\ln W_n = \sum_{t=1}^n \ln \frac{W_t}{W_{t-1}}~.
\]
The upper bound is obtained, as in \cite[Theorem~2.2]{CeLu06}, by Hoeffding's inequality,
\[
\ln \frac{W_t}{W_{t-1}} \leq - \eta \int_{\cA} \ell_t(I) \,\d\bp_t(I) + \frac{\eta^2 M^2}{8}~.
\]
A lower bound can be proved with techniques similar to the ones
appearing in ~\cite{BlKa97}, see also \cite[page~49]{CeLu06}.  We
denote by $I^*$ the element of $\cA$ achieving the infimum in the
definition of the regret (if it does not exist, then we take an
element of $\cA$ whose cumulative loss is arbitrarily close to the
infimum). As indicated in Section~\ref{sec:descrcont}, $I^*$ can be
described by the (ordered) shifting times $g_1^*,\ldots,g_m^*$ and the
corresponding actions $x_{k_1^*},\ldots,x_{k_{m+1}^*}$. We denote by
$\lambda$ the Lebesgue measure.  We consider the set of the
simultaneous actions $I$ that differ from $I^*$ on a union of
intervals of total length at most $\varepsilon > 0$, for some
parameter $\varepsilon > 0$,
\[
\cA_\varepsilon(I^*)
    = \bigl\{ I : \lambda\{ I \ne I^* \} \leq \varepsilon \bigr\}~.
\]
$\cA_\varepsilon(I^*)$ contains in particular the $I$ that can be
described with the same $m+1$ actions as $I^*$ and for which the
shifting times $g_1,\ldots,g_m$ are such that
\[
\sum_{j=1}^m \bigl| g_{(j)} - g_j^* \bigr| \leq \varepsilon~,
\]
i.e., the $I$ for which the corresponding probability distribution
$V(g_1,\ldots,g_m)$ as defined in~(\ref{eq:vect})
is $\varepsilon$--close in $\ell^1$--distance to
$V \bigl( g^*_1,\ldots,g^*_m \bigr)$.
Because $\mu$ induces by construction, via the application $V$, the uniform
distribution over the simplex of probability distributions over $m+1$ elements,
we get, by taking also into account the choice of the fixed $m+1$ actions of
$I^*$,
\[
\mu \bigl( \cA_\varepsilon(I^*) \bigr) \geq \frac{\varepsilon^m}{N^{m+1}}~.
\]
Here, we used the same argument as in~\cite{BlKa97}, based on
observing the fact that the uniform measure of the
$\varepsilon$--neighbor-hood of a point in the simplex of probability
distributions over $d$ elements equals $\epsilon^{d-1}$.  In addition,
because the $\psi_t$ take values in $[0,1]$, we have, for all
$I \in \cA_\varepsilon(I^*)$ and all $s \geq 1$,
\[
\ell_s(I) \leq \ell_s(I^*) + \lambda\{ I \ne I^* \} \leq \ell_s(I^*) + \varepsilon~.
\]
Putting things together, we have proved
\begin{eqnarray*}
\lefteqn{ \ln W_n } \\
& = & \ln \int_\cA \exp \left( - \eta \sum_{s=1}^{n} \ell_s(I) \right) \,\d\mu(I) \\
& \geq & \ln \, \left( \mu \bigl( \cA_\varepsilon(I^*) \bigr) \exp \left( - \eta \left( \varepsilon n + \sum_{s=1}^{n}
\ell_s(I^*) \right) \right) \right) \\
& \geq & - \eta \sum_{s=1}^{n} \ell_s(I^*) - \left(
m \ln \frac{1}{\epsilon} + (m+1) \ln N + \eta \varepsilon n \right)~.
\end{eqnarray*}
Combining the upper and lower bounds on $\ln W_n$ and substituting the proposed value for
$\eta$ concludes the proof.
\end{proof}
\medskip

Efficient implementation in this context requires exact simulation of
a step function $I$ according to (\ref{eq:pt-cont}), that is,
from the distribution
\begin{equation}
\label{eq:pt-cont2}
\d\bp_t(I) \, \propto \, \exp \left( - \eta \int_0^1 \phi_{t-1} \bigl( g, I(g) \bigr)
\d g \right) \, \d\mu(I)~
\end{equation}
for the functions defined, for each $x \in \cX$, as
\[
\phi_{t-1}(\,\cdot\,,x) = \sum_{s=1}^{t-1} \psi_s(\,\cdot\,,x)~,
\]
which take values in $[0,t-1]$. One could simulate from (\ref{eq:pt-cont2}) by
rejection sampling proposing from $\mu$; the probability of acceptance
is bounded below by something of the order of $e^{-\sqrt{t}}$, in view of the value of $\eta$.
Therefore, the computational cost of
such an algorithm, although only linear in $m$ and $N$, would be
typically exponential in $t$, hence unappealing.

Note that the problem (at each round $t$) can be represented as a discrete-time
Markov model. The Markov chain $Z$ is given by the pairs formed by the shifting times
and their corresponding actions,
$Z_j = \bigl( G_{(j)},K_{j+1} \bigr)$, for $j=0,\ldots,m$ and with the convention $G_{(0)}=0$.
Let $\pi$ denote the law of this Markov chain when
the times $G_1,\ldots,G_m$ are {i.i.d.} uniform over $[0,1]$ and
the action indexes $K_1,\ldots,K_{m+1}$ are taken {i.i.d.} uniform in $\{ 1,\ldots,N \}$.
Then simulating $I$ according to (\ref{eq:pt-cont2}) is
equivalent to simulating $Z$ according to the distribution
\[
\d \tilde{\pi}_{t-1}(Z) \, \propto \,
\prod_{j=2}^{m+1} w_j \bigl( K_{j-1},G_{(j-1)},G_{(j)} \bigr) \,\d \pi(Z)
\]
where, for $g \leq g'$,
\[
w_j(k,g,g') = \exp \left(- \eta \int_{g}^{g'} \phi_{t-1}(u,x_k) \, \d u \right )~,
\]
Exact simulation from $\tilde{\pi}_{t-1}$ is feasible when the
state-space of $Z$ is finite, and consists, e.g., in the same type
of dynamic programming approach discussed in Section~\ref{sec:effHMM}.
However, this is not the case here, since the second
component of $Z_j$ takes values in $[0,1]$. Approximating the
state-space of $Z$ by a grid is a possibility for an approximate
implementation, but it will be typically less efficient than the
approximation we advocate in Section~\ref{sec:appr-gener-discr}.

An interesting alternative is to resort to
sequential Monte Carlo methods (broadly known as particle
filters, see for example \cite{doucet} for a survey).
This is a class of methods ideally suited for approximating Feynman-Kac
formulae; a concrete example is the computation of expectations of bounded functions with
respect to the laws $\tilde{\pi}_{t-1}$ defined above. This is achieved by generating
a swarm of a given large number of weighted particles. The generation of particles
is done sequentially in $j=1,\ldots,m+1$ by importance sampling, and it involves interaction
of the particles at each step. This generates an interacting particle
system whose stability properties are well studied (see, for
instance, \cite{delmoral}). Resampling a single element from the
particle population according to the weights gives as an approximate
sample from $\tilde{\pi}_{t-1}$, hence from (\ref{eq:pt-cont2}). The total variation
distance between the approximation and the target is typically
$C(m+1)/K$, for some constant $C$ depending on the range of the integrands.
In the most naive implementation in this context,
one might thus have that $C$ is exponentially small in $t/m$.
The idea of an on-going work would be to make $C$ independent of $t$ by carefully
designing the importance sampling at each step taking into account the
characteristics of the $\phi_{t-1}$.
\medskip

% GS : Dropped for now.
% Finally, we note that (\ref{eq:pt-cont2}) is related to the law of
% continuous-time Markov chain with state-space $\cX$, and
% (non-homogeneous) exit rates
% when entering at state $x_j$, $\phi(g,j)$, conditioned to have exactly
% $m$ transitions in $[0,1]$ (note that $\mu$ can be seen as the measure
% of a conditioned continuous-time Markov chain). We can simulate
% exactly paths according to the conditioned Markov process and use them
% as proposals for rejection sampling from (\ref{eq:pt-cont2}).

Below we use a simple discretization and apply the techniques of previous
sections to achieve approximate sampling from (\ref{eq:pt-cont}).

\subsection{Approximate generation by discretization}
\label{sec:appr-gener-discr}

Here we show how an approximate version of
the forecaster described above can be implemented efficiently.

The argument works by partitioning $[0,1]$ into intervals
$G^0 = [0,1/\varepsilon[$, $G^1 = [1/\varepsilon, 2/\varepsilon[$,
$\ldots$, $G^{M_\varepsilon}$ of length $\varepsilon$ (except
maybe for the last interval of the partition), for some fixed
$\varepsilon > 0$, and using the same action for all tasks in
each $G^{j}$. Here, we aggregate all tasks within an interval
$G^{j}$ into a super-task $j$.  We have
$M = M_\varepsilon = \lceil 1/\varepsilon \rceil$ of these super-tasks
and will be able to apply the techniques of the finite case.

More precisely, we restrict our attention to the elements of $\cA$
whose shifting times (in number less or equal to $m$) are starting
points of some $G^{j}$, that is, are of the form $j/\varepsilon$ for
$0 \leq j \leq M_\varepsilon$.  We call them simultaneous actions
compatible with the partitioning and denote by $\cB_\varepsilon$ the
set formed by them.  The loss of super-task $j$ at time $t$ given the
simultaneous actions described by the element $I \in \cB_\varepsilon$
is denoted by
\[
\ell^{(j)}_t(I) = \int_{G^j} \psi_t \bigl( g, I(j/\varepsilon) \bigr) \,\d g~.
\]
Note that these losses satisfy $\ell^{(j)}_t(I) \in [0,\varepsilon]$.

By the same argument as the one used in the proof of
Theorem~\ref{th:cont}, we have
\[
\inf_{I \in \cA} \sum_{t=1}^n \ell_t(I) \leq
\inf_{I \in \cB_\varepsilon} \sum_{t=1}^n \ell_t(I) + \frac{m n \varepsilon}{2}~.
\]
This approximation argument, combined with Proposition~\ref{th:full} and the results of Section~\ref{sec:effHMM}
leads to the following. (We use here the fact that there are not more than
\[
\nchoosek{M_\varepsilon}{m} N^m \leq \bigl( M_\varepsilon N)^m
\]
elements in $\cB_\varepsilon$.

\begin{theorem}
For all $\varepsilon > 0$, the weighted average
forecaster run on the $M_\varepsilon$ super-tasks defined above, under
the constraint of not more than $m$ shifts, ensures that for a proper
choice of $\eta$ and with probability at least $1-\delta$, the regret
is bounded as
\[
R_n \leq \sqrt{\frac{n m \ln \bigl( N \lceil 1/\varepsilon \rceil \bigr)}{2}}
+ \frac{m n \varepsilon}{2} + \sqrt{\frac{n}{2} \ln \frac{1}{\delta}}
\]
In addition, its complexity of implementation is
$O \bigl((Nm)^2/\epsilon \bigr)$.
\end{theorem}

The choice of $\varepsilon$ of the order of $1/\sqrt{n}$ yields a
bound comparable to the one of Theorem~\ref{th:cont}, for a moderate
computational cost of $O \bigl( \sqrt{n} (Nm)^2 \bigr)$.

These results can easily be extended
to the bandit setting, when $\psi_t$ is only observed
through $I_t$ as
\[
\ell_t(I_t) = \int_{[0,1]} \psi_t \bigl( g,I_t(g) \bigr) \,\d g~.
\]
This is because whenever $I_t$ is compatible with the partitioning,
the latter is also the sum of the losses of the actions taken in each of the super-tasks.
The techniques of Section~\ref{sec:bandits} can then be applied again.

\bibliography{Bib-subgames}
\bibliographystyle{alpha}

\end{document}